\documentclass[11pt]{article}

\usepackage{fullpage}

\usepackage{amsthm,amsfonts,amsmath,amssymb,epsfig,color,float,graphicx,verbatim}
\usepackage{algorithm,algorithmic}

\usepackage{cases}

\usepackage{hyperref}
\hypersetup{
colorlinks=true,
citecolor=blue,
linkbordercolor={1  0 1}, 
citebordercolor={1 1 1} 
}

\usepackage{algorithm,algorithmic}

\newtheorem{thm}{Theorem}

\newtheorem{lemma}{Lemma}

\newtheorem{definition}{Definition}

\DeclareMathOperator*{\argmin}{argmin}

\DeclareMathOperator*{\sgn}{sign}
\DeclareMathOperator*{\st}{s.t.}

\def \R {\mathbb{R}}
\def \x {\mathbf{x}}
\def \E {\mathrm{E}}

\def \B {\mathcal{B}}

\def \u {\mathbf{u}}

\def \y {\mathbf{y}}

\def \w {\mathbf{w}}

\def \v {\mathbf{v}}

\def \D {\mathcal{D}}
\def \H {\mathcal{H}}

\def \P {\mathcal{P}}

\def \dis {\mbox{DIS}}
\def \X {\mathcal{X}}
\def \C {\mathcal{C}}

\def \ep {\ell_{\phi}}
\def \ebin {\ell_b}

\DeclareSymbolFont{bbold}{U}{bbold}{m}{n}
\DeclareSymbolFontAlphabet{\mathbbold}{bbold}
\def \ind {\mathbbold{1}}

\title{Beating the Minimax Rate of Active Learning with Prior Knowledge}
\author{Lijun Zhang \qquad Mehrdad Mahdavi \qquad Rong Jin  \\
Department of Computer Science and Engineering \\ Michigan State University \\
\texttt{\{zhanglij,mahdavim,rongjin\}@msu.edu}}

\date{}
\begin{document}

\maketitle
\begin{abstract}
Active learning refers to the learning protocol where the learner is allowed to choose a subset of instances for labeling. Previous studies have shown that, compared with passive learning, active learning is able to reduce the label complexity exponentially if the data are linearly separable or satisfy the Tsybakov noise condition with parameter $\kappa=1$. In this paper, we propose a novel active learning algorithm using a convex surrogate loss, with the goal to broaden the cases for which active learning achieves an exponential improvement. We make use of a convex loss not only because it reduces the computational cost, but more importantly because it leads to a tight bound for the empirical process (i.e., the difference between the empirical estimation and the expectation) when the current solution is close to the optimal one. Under the assumption that the norm of the optimal classifier that minimizes the convex risk is available, our analysis shows that the introduction of the convex surrogate loss yields an exponential reduction in the label complexity even when the parameter $\kappa$ of the Tsybakov noise is larger than $1$. To the best of our knowledge, this is the first work that improves the minimax rate of active learning by utilizing certain priori knowledge.
\end{abstract}

\section{Introduction}
The goal of active learning is to reduce the number of training examples required for a learner to achieve good generalization performance~\cite{Cohn:2011:active}. In~\cite{Castro:2007:MBA}, the authors show that the minimax convergence rate for any active learning algorithm is bounded by $n^{-\frac{\kappa}{2\kappa - 2}}$, where $n$ is the number of labeled instances and $\kappa \geq 1$ is used in Tsybakov noise condition~\cite{Tsybakov04optimalaggregation} to characterize the behavior of $\Pr(Y = 1|X = \x)$ in the neighborhood of the decision boundary.\footnote{We omit an additional parameter $\rho$ in the minimax rate, that is because $\rho=0$ under the assumption that the optimal classifier is a linear function.} This result implies that unless $\kappa = 1$, no active learning algorithm is able to achieve an exponential reduction in the label complexity for general data distributions. In this work, we develop theory and algorithm for active learning that aim to override this impossibility. We show that if $R = \|\w_*\|$ is known apriori, where $\w_*$ minimizes both the binary  risk and the convex risk, it is possible to achieve an exponential reduction in label complexity for certain family of distributions even when $\kappa > 1$. We emphasize that our result does not contradict with the minimax rate of active learning proved in~\cite{Castro:2007:MBA} because we assume the prior knowledge of $R = \|\w_*\|$ is available to the learner.

Our work is closely related to the previous study of surrogate losses in passive and active learning~\cite{Hanneke:2012:Surrogate}, in which the authors showed that under appropriate conditions, the exponential reduction in label complexity can still be achieved when the binary loss function is replaced with a convex loss in order to improve the computational efficiency in active learning. In this study, we make one step further. We show that besides the computational efficiency, the introduction of convex loss function could also benefit the convergence rates of active learning if the length of the optimal linear classifier for the convex risk is known a priori. The key idea is to explore the Lipschitz smoothness of the convex loss function and the technique of local Rademacher complexity~\cite{Local_RC,Local:Vladimir}, i.e., the concentration bound gets tighter as the solution is approaching the optimal one. It is the improved concentration bound, due to the use of a convex loss function, that leads to the surprising result that under the favored conditions, including the assumption that $\|\w_*\|$ is known, it is possible to achieve an exponential reduction in label complexity even when $\kappa > 1$.

\section{Related Work}
According to~\cite{gonen:2012:efficient}, most active learning algorithms can be classified into two categories: greedy algorithms and mellow algorithms. The greedy active learning algorithms are designed to select the most informative instances for labeling that will result in an approximately even split of the hypothesis space. Instead of trying to find the most informative instances for labeling, the mellow algorithms for active learning, sometimes referred to as selective sampling, solicit labels for instances as long as they satisfy a given criterion which is usually adjusted over iterations.  In this work, we will focus on selective sampling for active learning.

Many studies show that active learning can result in exponential reduction of label complexity compared to passive learning when data are linearly separable~\cite{Freund:1997:SSU,Dasgupta:2005:Coarse,Dasgupta:2009:APA}. More recent studies focus on agnostic active learning where data cannot be perfectly classified by a linear classifier~\cite{Balcan:2006:AAL,Balcan:2007:margin,balcan:2010:true,beygelzimer:2010,Castro:2007:MBA,Hanneke:2010:Negative,Hanneke:2011:Rates,Hanneke:2012:Surrogate,Koltchinskii:2010:RCB}.
In~\cite{Awasthi:2013,balcan:2013:active}, the authors show that active learning can achieve exponential reduction in label complexity for the malicious noise model, the adversarial nose model, and the Tsybakov noise model with parameter $\kappa=1$. The minimax rates for active learning under the Tsybakov noise model is discussed in~\cite{Castro:2007:MBA}, which shows it is in general impossible to reduce the sample complexity exponentially when the parameter $\kappa >1$. In this study, we try to override the minimax rate for the Tsybakov noise condition by considering the scenario when $R = \|\w_*\|$ is know apriori, where $\w_*$ minimizes both the binary risk and the convex risk.

Various algorithms have been developed for active learning. Several active learning algorithms~\cite{Balcan:2006:AAL,Hanneke:2011:Rates} require maintaining the subset of hypotheses that yield small binary excess risk, which may result in a high computational cost. This limitation was addressed by either making specific assumptions about data distribution~\cite{Balcan:2007:margin,balcan:2013:active} or by maintaining only two hypotheses and soliciting the class label for an instance $\x$ when the two hypotheses differ significantly~\cite{dasgupta:2007:general,beygelzimer:2010}. Several studies extend active learning theory to the online learning setup~\cite{Dasgupta:2009:APA,Cavallanti:2011:LNL,dekel:2012:selective}.

An important quantity in the analysis of active learning is disagreement coefficient~\cite{Hanneke:2007:Bound,Hanneke:2011:Rates}. It was shown in~\cite{Koltchinskii:2010:RCB} that the disagreement coefficient is closely related to the capacity function~\cite{alexander:1987:rates}, and under the Massart low noise condition, the capacity function can be bounded by a constant. It was further shown in~\cite{Friedman:2009:Active}, the disagreement coefficient can be bounded by a constant for any smooth function. In~\cite{raginsky:2011:lower}, the capacity function was used to provide lower bounds for both passive and active learning.

The convex surrogate loss has been well-studied in the passive learning to ease the computational problems~\cite{zhang2004,ICML2012Ben-David_917}. In particular, our analysis is heavily built upon the theories for converting convex excess risk to binary excess risk developed in~\cite{bartlett-2006-convexity}. The study of convex surrogate loss in active learning is still in its early stages, and only limited results are available. In~\cite{Hanneke:2012:Surrogate}, the authors present an active learning algorithm based on a surrogate loss. Following this line of research, in this paper, we show that the introduction of convex surrogate loss not only reduces the computational cost of active learning, but also improves the sample complexity if some prior knowledge is known. We note that our conclusion does not conflict with the negative results in~\cite{Hanneke:2010:Negative} (i.e. the basic approach of optimizing the surrogate risk via active learning to a sufficient extent to guarantee small error rate generally does not lead to as strong of results.) , because we make some additional assumptions, in particular the assumption that $\|\w_*\|$ is available to the learning process.

\section{Assumptions}
We first introduce the notations used in this paper, and discussion the assumptions we made.
\subsection{Notations}
Let $\X \subseteq \R^d$ be the domain for the input patterns for classification, and $\mathcal{Y} = \{-1, +1\}$ be the binary class assignment. Let $\P_{XY}$ be the joint distribution for input pattern $X \in \X$ and output binary class assignment $Y \in \mathcal{Y}$, and $\P_X$ be the  marginal distribution for $\X$. Let $\phi(z)$ be a convex loss function that $L$-Lipschitz continuous.  For a prediction function $f(\x): \R^d \rightarrow \R$, we define the convex risk $\ep(\cdot)$ and binary risk $\ebin(\cdot)$ as
\[
\begin{split}
    \ep(f)   & = \E_{(\x, y)\sim \P_{XY}}[\phi(y f(\x)], \\
    \ebin(f) & = \E_{(\x, y)\sim \P_{XY}}[\ind(y f(\x) \leq 0)],
\end{split}
\]
where $\ind(z)$ is an indicator function that outputs $1$ when the predicate $z$ is true and zero otherwise.

In this study, we focus on linear classifier like most studies of active learning.  To keep the notation simple, we refer to a weight vector and the linear classifier with that weight vector interchangeably. Let $\w_*$ be the optimal linear classifier that minimizes the convex risk $\ep(\cdot)$, i.e.,
\[
    \w_* = \argmin\limits_{\w \in \R^d} \ep(\w).
\]
Throughout the paper, we will use $\bar{\w}$ to represent the unit vector that is along the direction of $\w$, and use $\|\w\|$ to represent the $\ell_2$-norm of vector $\w$.

\subsection{Assumptions}
In this following, we discuss the assumptions made about the optimal solution $\w_*$ and the binary risk $\ebin(\w)$ that are crucial to our analysis. They are:
\begin{itemize}
\item {\bf Assumption (I)}:  $R=\|\w_*\|$ is bounded and known apriori to the learner.
\item {\bf Assumption (II)}: $\w_*$ minimizes both the convex risk $\ep(\cdot)$ and the binary risk $\ebin(\cdot)$ over all the measurable functions.
\item {\bf Assumption (III)}: There exists constants $\ell_-, \ell_+ >0$ and $\gamma_-  \geq \gamma_+ >0$ such that for any $\w \in \R^d$
\[
    \ell_-\|\bar{\w} - \bar{\w}_*\|^{\gamma_-} \leq \ebin(\w) - \ebin(\w_*) \leq \ell_+\|\bar{\w} - \bar{\w}_*\|^{\gamma_+}
\]
\end{itemize}
To quantify the noise level, we choose the Tsybakov low noise condition~\cite{Tsybakov04optimalaggregation}, which combined with Assumption (II) leads to the following condition.
\[
\Pr_{X \sim \P_X} \left\{\sgn(\w^{\top}X) \neq \sgn(\w_*^{\top}X) \right\} \leq \mu (\ebin(\w) - \ebin(\w_*))^{1/\kappa}
\]
for some constants $\mu>0$ and $\kappa \geq 1$.

Assumptions (I) assumes the prior knowledge of $R = \|\w_*\|$. The knowledge of $R$ may be obtained based on the assumption of the application domain. For instance, in the case of transfer learning~\cite{pan-2010-survey}, when instances in the target domain are subjected to a unknown unitary transformation of instances from a source domain, the length of the linear classifier will be preserved and transferred from the source domain to the target domain, leading to the knowledge of $R$ for the target domain. This type of transfer learning problem appears in computer vision~\cite{graepel-2003-invariant}, where images of one domain are acquired by applying certain invariant transform to the images from another domain.

The combination of Assumptions (I) and (II) allows us to nicely connect the minimization of the binary loss with the minimization of a convex loss. It is this connection that makes it possible to improve the convergence rate of active learning. Assumption (III) is a key technical assumption to our analysis. It is introduced to ensure that as the binary risk is reduced, the estimated solution is getting closer to the optimal solution, which makes it possible to explore the local Rademacher complexity of minimizing the convex loss for faster convergence rate of active learning. Below we will justify Assumptions (II) and (III).

\subsection{Justification for Assumption (II)} \label{eqn:Ass:II}
Define
\[
\eta(\x) = \Pr(Y=1|X=\x).
\]
The optimal prediction function $\tau(\cdot)$ that minimizes $\ep(\cdot)$ over all measurable function is given by
\[
\tau(\x) = \argmin\limits_{z \in \R} \eta(\x) \phi(z) + (1 - \eta(\x)) \phi(-z).
\]
The first part of Assumption (II) (i.e. $\w_*$ minimizes $\ell_{\phi}(\cdot)$) assumes that $\tau(\cdot)$ is a linear function, which is also used in the recent study of the convex surrogate loss~\cite{Hanneke:2012:Surrogate}. This assumption allows us to bound the binary excess risk in terms of convex excess risk~\cite{bartlett-2006-convexity}.

\paragraph{Remark} There are some special cases of $\eta(\cdot)$ and $\phi(\cdot)$ such that $\tau(\cdot)$ is certainly a linear function~\cite{bartlett-2006-convexity}.
\begin{itemize}
\item Exponential loss $\phi(\alpha) = e^{-\alpha}$, and a logistic model $\eta(x)=1/[1+\exp(-\w^\top \x)]$. We have
\[
\tau(\x)=\frac{1}{2} \log\left(\frac{\eta(\x)}{1-\eta(\x)} \right) = \frac{1}{2} \w^\top \x.
\]
\item Truncated quadratic loss $\phi(\alpha) = [\max(0, 1 - \alpha)]^2$, and an affine model $\eta(x)=\w^\top \x+1/2$. We have
\[
\tau(\x)=2\eta(\x)-1= \w^\top \x.
\]
\end{itemize}

The second part of this assumption (i.e. $\w_*$ also minimizes the binary loss $\ell_b(\cdot)$) is a direct consequence of the first one if the convex loss is \emph{classification-calibrated}~\cite{bartlett-2006-convexity}.
\begin{definition}
A convex loss $\phi(\cdot)$ is classification-calibrated if, for any $\eta \neq 1/2$,
\[
H^-(\eta) > H(\eta),
\]
where
\[
H^-(\eta)=\inf_{\alpha:\alpha(2\eta-1)\leq 0}\left( \eta \phi(\alpha)+(1-\eta) \phi(-\alpha)\right), \textrm{ and } H(\eta)= \inf_{\alpha\in\R} \left( \eta \phi(\alpha)+(1-\eta) \phi(-\alpha)\right).
\]
\end{definition}


We finally note that $\tau(\x)$ only depends on conditional distribution $\eta(\x)$ and the convex loss $\phi(\cdot)$, and is independent from the marginal distribution $\P_X$. Thus, Assumption (II) actually holds for any distribution of $X$, as long as $\eta(\x)$ and $\phi(\cdot)$ remain the same.

\subsection{Justification for Assumption (III)}
\label{sec:justificaiton-III}
This assumption will be used in our analysis to bound the distance between $\bar{\w}$ and $\bar{\w}_*$ using the difference in their binary risk. We first examine the lower bound in Assumption (III), i.e., $\ebin(\w) - \ebin(\w_*) \geq \ell_- \|\bar{\w} - \bar{\w}_*\|^{\gamma_-}$, and then discuss the upper bound $\ebin(\w) - \ebin(\w_*) \leq \ell_+ \|\bar{\w} - \bar{\w}_*\|^{\gamma_+}$.

\paragraph{Lower bound in Assumption (III)} The following lemma bounds the values for $\ell_-$ and $\gamma_-$ when $X$ follows an isotropic log-concave distribution.
\begin{lemma} \label{lem:lower}
Assume $\P_X$ is an isotropic log-concave in $\R^d$ and assume that the Tsybakov condition holds with constants $\mu>0$ and $\kappa \geq 1$. We have $\gamma_- = \kappa$ and $\ell_- \geq \frac{c^\kappa}{\mu^\kappa}$ for Assumption (III), where $c$ is an universal constant defined in~\cite[lemma 3]{balcan:2013:active}.
\end{lemma}

\begin{proof}
We need the following lemma regarding the isotropic log-concave distribution~\cite{balcan:2013:active}.
\begin{lemma}(\cite[Lemma 3]{balcan:2013:active})
Assume $\P_X$ is an isotropic log-concave in $\R^d$. Then, there exists constant $c > 0$ such that for any two unit vectors $\u$ and $\v$ we have
\[
    c\theta(\u, \v) \leq \Pr_{X \sim \P_X} \left\{\sgn(\u^{\top}X) \neq \sgn(\v^{\top}X) \right\},
\]
where $\theta(\u, \v)$ is the angle between $\u$ and $\v$.
\end{lemma}
Using the fact $\sin(x) \leq x$ and the above lemma, we have
\begin{eqnarray*}
\lefteqn{\left\|\bar{\w} - \bar{\w}_*\right\| = \sqrt{2- 2 \cos\left(\theta(\w, \w_*) \right) } = 2\sin\left(\frac{1}{2}\theta(\w, \w_*)\right)}\\
& & \leq \theta(\w, \w_*) \leq \frac{1}{c}\Pr_{X \sim \P_X} \left\{\sgn(\w^{\top}X) \neq \sgn(\w_*^{\top}X) \right\}.
\end{eqnarray*}
From the Tsybakov low noise condition, we have
\[
\left\|\bar{\w} - \bar{\w}_*\right\| \leq \frac{\mu}{c}(\ebin(\w) - \ebin(\w_*))^{1/\kappa}.
\]
\end{proof}

\paragraph{Upper bound in Assumption (III)} The following lemma justifies the upper bound when $\P_X$ is orthogonally invariant.\footnote{In the literature, orthogonally invariant is also refereed to as isotropic~\cite{Isotropic}, which is different from the definition of isotropic in~\cite{balcan:2013:active}.}
\begin{lemma} \label{lem:upper:new}
Suppose $\P_X$ is orthogonally invariant. We have $\gamma_+= 1$ and $\ell_+ = 1/2$ for Assumption (III).
\end{lemma}
\begin{proof}
First, we have
\[
\begin{split}
\ebin(\w) - \ebin(\w_*) \leq \Pr_{X \sim \P_X} \left\{\sgn(\w^\top X) \neq \sgn(\w_*^\top X) \right\}.
\end{split}
\]
Since the sign function is invariant respect to scaling, we have
\[
\ebin(\w) - \ebin(\w_*)  \leq \Pr_{X \sim \P_X} \left\{\sgn(\w^\top X/\|X\|_2) \neq \sgn(\w_*^\top X/\|X\|_2) \right\}.
\]
Define $Y=X/\|X\|_2$. Because the distribution of $X$ is orthogonally invariant, it is well-known that $Y$ follows the uniform distribution on the $n$-dimensional sphere, which is denoted by $\sigma_n$~\cite{Isotropic}. Following~\cite[Lemma 3.2]{Goemans:1995:IAA}, we have
\begin{equation} \label{eqn:bound:iso}
\ebin(\w) - \ebin(\w_*)  \leq \Pr_{Y \sim \sigma_n} \left\{\sgn(\w^\top Y) \neq \sgn(\w_*^\top Y) \right\}= \frac{1}{\pi} \theta(\w, \w_*).
\end{equation}
From Jordan's inequality, we know $\sin(x) \geq 2 x/\pi$ for $x \in [0,\pi/2]$. Thus, we have
\[
\left\|\bar{\w} - \bar{\w}_*\right\| =  2\sin\left(\frac{1}{2}\theta(\w, \w_*)\right)  \geq  \frac{2}{\pi}  \theta(\w, \w_*) \overset{\text{(\ref{eqn:bound:iso})}}{\geq} 2 \ebin(\w) - \ebin(\w_*).
\]
\end{proof}
A more general discussion for the upper bound in Assumption (III) can be found in the appendix.

\section{Algorithm}
Our algorithm works as follows. We divide the learning into $m$ epoches. At the $k$th epoch, we have a hypothesis space $\Omega_k = \{\w \in \R^d: \|\w\|=1, \|\w - \w_k\| \leq r_k=2^{-k+2}\}$, where $\w_k$ is an unit vector that is computed from the previous epoch, and $r_k$ specifies the size of the domain. We sequentially scan through the pool of training examples, and request the class label for a training instance only when it interacts with the domain $\Omega_k$. More specifically, we will request the class label for instance $\x$ if
\begin{equation} \label{eqn:selection-criterion}
\sgn(\w^{\top}\x) \neq \sgn(\w^{\top}_k \x), \ \exists  \w \in \Omega_k.
\end{equation}
The following theorem simplifies this condition significantly.
\begin{lemma} \label{lem:selection}
The condition in (\ref{eqn:selection-criterion}) is equivalent to
\begin{subnumcases}{|\bar{\x}^{\top}\w_k| \leq}
r_k \sqrt{1-r^2_k/4}, & if $r_k \leq 1$; \label{eqn:selection1} \\
1,  & if $r_k=2$. \label{eqn:selection2}
\end{subnumcases}
Here, $\bar{\x}$ is the unit vector along the direction of $\x$.
\end{lemma}

We denote by $\D_k = \{(\x_k^t, y_k^t), t=1, \ldots, n_k\}$ the collection of labeled training examples received at epoch $k$, where $n_k$ is the number of labeled instances at epoch $k$. Using the training examples in $\D_k$, we compute a new classifier $\w_{k+1}$ as the solution to
\begin{equation} \label{eqn:update}
\min\limits_{\w \in \Omega_k} \sum_{t=1}^{n_k} \ind \left(y_k^t \neq \sgn(\w^{\top}\x_k^t)\right).
\end{equation}
The new hypothesis space, denoted by $\Omega_{k+1}$, is then updated as
\[
\Omega_{k+1} = \left\{\w \in \R^d: \|\w\| = 1, \|\w - \w_{k+1}\| \leq r_{k+1} = r_k/2 \right\},
\]
where the size of the hypothesis space is reduced by half through each epoch. We note that the idea of reducing the hypothesis space by half for each epoch has been used in the margin-based active learning algorithm~\cite{Balcan:2007:margin,balcan:2013:active}. The main difference between this work and the previous ones is that we update the solution by minimizing a convex surrogate loss, a key component that allows us to improve the minimax rate under favored conditions.

One problem with the updating procedure given in (\ref{eqn:update}) is that it requires solving a non-convex optimization problem, which could be computationally expensive when the number of training examples is large. To address this limitation, we propose to obtain $\tilde{\w}_{k+1}$ by solving the following convex optimization problem
\begin{eqnarray} \label{eqn:convex-update}
\min \limits_{\w \in \Delta_k} \sum_{t=1}^{n_k} \phi(y_k^t\w^{\top}\x_k^t),
\end{eqnarray}
where $\Delta_k = \left\{\w \in \R^d: \|\w - R\w_k\| \leq R r_k\right\}$. Here $R = \|\w_*\|$ comes from the prior knowledge of $\w_*$ due to Assumption (I). The final $\w_{k+1}$ is obtained by normalizing $\tilde{\w}_{k+1}$ to a unit vector. Algorithm~\ref{alg:1} summarizes the key steps of both approaches.

\begin{algorithm}[t]
\caption{Active Learning with Faster Convergence Rate}

\begin{algorithmic}[1]
\STATE Set $\w_1$ as a random unit vector and $r_1 = 2$
\FOR{$k=1, 2, \ldots, $m}
    \STATE Label $n_k$ training instances that satisfy (\ref{eqn:selection1}) or (\ref{eqn:selection2}).
    \STATE Learn a new classifier $\tilde{\w}_{k+1}$ solving either the non-convex optimization problem in (\ref{eqn:update}) or the convex optimization in (\ref{eqn:convex-update})
    \STATE Set $\w_{k+1} = \tilde{\w}_{k+1}/\|\tilde{\w}_{k+1}\|$ and $r_{k+1} = r_k/2$.
\ENDFOR
\end{algorithmic} \label{alg:1}
{\bf Return} $\x_f$
\end{algorithm}

\section{Analysis}
We will first introduce the basic concepts that are commonly used in the analysis of active learning. We will then analyze the label complexity for solving the non-convex optimization problem in (\ref{eqn:update}). The key result of this work is presented in Section~\ref{sec:convex-analysis}, where we show that the exponential reduction can be achieved even when $\kappa > 1$ if we solve the convex optimization problem in (\ref{eqn:convex-update}) in Algorithm~\ref{alg:1}. Due to space limitations, most of the technical proofs are provided in the Appendix.

\subsection{Basics}
Similar to most active learning theories, we assume bounded disagreement coefficient~\cite{Hanneke:2011:Rates}. We define the region of disagreement as, for any subset of hypothesis $V$,
\[
\dis(V) = \left\{\x \in \X: \exists \w_1, \w_2 \in V \st \sgn(\w_1^{\top}\x) \neq \sgn(\w_2^{\top}\x)\right\}.
\]
For $r \in [0, 1]$, let
\[
B(\w, r) = \left\{ \w' \in \B: \Pr\left\{\sgn(\w^{\top}X) \neq \sgn([\w']^{\top}X)\right\} \leq r\right\},
\]
where
\[
\B=\{ \w \in \R^d: \|\w\|=1 \}.
\]
The disagreement coefficient of $\w$ with respect to $\B$ is then defined as
\[
\theta_{\w}(\varepsilon) = \sup\limits_{r \geq \varepsilon}\frac{\Pr(\dis(B(\w, r)))}{r}.
\]
Define
\[
\theta(\varepsilon) = \theta_{\w_*}(\varepsilon).
\]
Note that we keep the dependence of $\varepsilon$ in the definition of disagreement coefficient since it may include factor $\log(1/\varepsilon)$ as indicated in~\cite{balcan:2013:active}. The disagreement coefficient allows us to connect the largest disagreement between two classifiers in a given hypothesis space with the percentage of the examples that are classified differently by at least two classifiers in the hypothesis space.

Since our work tries to bound the binary excess risk with a convex excess risk, we need the $\psi$-transform~\cite{bartlett-2006-convexity} stated below,
\[
\psi(z) = \inf\limits_{\alpha z \leq 0}\left(\frac{1 + z}{2}\phi(\alpha) + \frac{1 - z}{2}\phi(-\alpha)\right) - \inf\limits_{\alpha \in \R}\left(\frac{1 + z}{2}\phi(\alpha) + \frac{1 - z}{2}\phi(-\alpha)\right).
\]
Here are two examples of $\psi$-transform from~\cite{bartlett-2006-convexity}: (i) for exponential loss $\phi(\alpha)=e^{-\alpha}$, $\psi(z) = 1 - \sqrt{1 - z^2} \geq z^2/2$, and (ii) for truncated quadratic loss $\phi(\alpha) = [\max(0, 1 - \alpha)]^2$, $\phi(z) = z^2$.

The following theorem from~\cite[Theorem 1]{bartlett-2006-convexity} shows that the binary excess risk can be bounded by the convex excess risk using $\psi$-transform.
\begin{thm} \label{thm:classification-caliberated} For any non-negative loss function $\phi$, any measurable function $f:\X\mapsto \R$, and any probability distribution on $\X\times\{-1, +1\}$, we have
\[
\psi(R(f) - R^*) \leq R_{\phi}(f) - R_{\phi}^*
\]
where $R(f) = \E_{(X,Y)\sim \P_{XY}}\left[\ind (y f(X) \leq 0) \right]$, $R^* = \min_f R(f)$,$R_{\phi}(f) = \E_{(X, Y) \sim \P_{XY}}\left[ \phi(Yf(X))\right]$, and $R_{\phi}^* = \min_f R_{\phi}(f)$. Here the minimization is taken over all measurable functions.
\end{thm}
\subsection{Label Complexity for Non-Convex Optimization}
Our analysis is based on induction. The key to our analysis is to show that $\bar{\w}_* \in \Omega_{k+1}$, if (i) $\bar{\w}_* \in \Omega_k$ and (ii) $n_k$, the number of labeling queries issued at epoch $k$, is sufficiently large.
\begin{thm} \label{thm:induction}
Suppose Assumption (III) holds, and $\bar{\w}_* \in \Omega_k$. Then, with a probability $1 - \delta/m$, we have $\bar{\w}_* \in \Omega_{k+1}$, if
\begin{equation} \label{eqn:nk}
n_k = 2 c^2 \theta^2(\varepsilon) \left[ \log \frac{4m}{\delta}  + 2(d+1) \left( \log 8 + 2 \log   \frac{c \theta(\varepsilon)  }{ r_k^{\gamma_- - \gamma_+/\kappa} } \right)\right]  r_k^{2 \left( \frac{\gamma_+}{\kappa}-\gamma_-\right)},
\end{equation}
where $c = \mu \ell_+^{1/\kappa}2^{\gamma_0}/\ell_-$ and $\gamma_0 = 2+ \gamma_- + \gamma_+ /\kappa$.
\end{thm}

\begin{thm} \label{thm:complexity-1}
Suppose Assumption (III) holds. Let $\w_{m+1}$ be the solution output from the proposed algorithm after $m$ iterations, where $m = \lceil \log_2 (2/\varepsilon) \rceil$. Then, with a probability $1 - \delta$, we have
\[
    \left\|\w_{m+1} -\bar{\w}_*  \right\| \leq \varepsilon,
\]
and the total number of labeled instances is bounded by
\[
n \leq \begin{cases}
n_0 \frac{ 2^{2\alpha} }{2^{2\alpha}-1} \left( \frac{4}{\epsilon}\right)^{2\alpha},& \alpha > 0\\
 n_0 \log_2 \frac{4}{\varepsilon}, & \alpha \leq 0
\end{cases}
\]
where
\[
\begin{split}
\alpha&=\gamma_- - \frac{\gamma_+}{\kappa}, \\
n_0 & = 2^{1-4 \alpha}  c^2 \theta^2(\varepsilon) \left[ \log \frac{4m}{\delta}  + 2(d+1) \left( \log 8 + 2 \log c \theta(\varepsilon) + 2m \alpha \log 2 \right)\right].
\end{split}
\]
\end{thm}
\begin{proof}
From Theorem~\ref{thm:induction}, with a probability $1-\delta$,  we have
\[
 \left\|\w_{m+1}-\bar{\w}_* \right\| \leq r_{m+1} =  r_12^{-m} = 2^{-m+1} \leq \varepsilon.
\]
and the number of labeled instances is bounded by
\[
\begin{split}
n = & \sum_{k=1}^m n_k \\
=& \sum_{k=1}^m 2^{1-4 \alpha}  c^2 \theta^2(\varepsilon)   \left[ \log \frac{4m}{\delta}  + 2(d+1) \left( \log 8 + 2 \log c \theta(\varepsilon) + 2(k-2)\alpha \log 2 \right)\right] 2^{2 \alpha k} \\
\leq & n_0  \sum_{k=1}^m 2^{2 \alpha k}.
\end{split}
\]
In the case that $\alpha \leq 0$, we have
\[
n \leq n_0 m  \leq n_0 \log_2 \frac{4}{\varepsilon}.
\]
Otherwise, we have
\[
n \leq n_0 \frac{2^{2\alpha} (2^{2\alpha m}-1 ) }{2^{2\alpha}-1} \leq n_0 \frac{ 2^{2\alpha} }{2^{2\alpha}-1} 2^{2\alpha m} \leq n_0 \frac{ 2^{2\alpha} }{2^{2\alpha}-1} \left( \frac{4}{\epsilon}\right)^{2\alpha}.
\]
\end{proof}

\paragraph{Remark}  When the distribution $\P_X$ is isotropic log-concave and orthogonally invariant,  we have $\gamma_-=\kappa$ and $\gamma_+=1$, as discussed in Lemmas~\ref{lem:lower} and \ref{lem:upper:new}. In the case when $\kappa = 1$, $\alpha =0$ and Algorithm~\ref{alg:1} achieves exponential reduction in label complexity, consistent with the previous studies. In the next subsection, we show that it is possible to achieve exponential reduction in label complexity even $\kappa > 1$ provided Assumptions (I)-(III) hold.

\subsection{Label Complexity for Convex Optimization}
\label{sec:convex-analysis}
For the simplicity of the presentation, we first assume a bounded $\ell_2$-norm for the input $X$, i.e., $\|X\| \leq 1$, and discuss a relaxed condition in the end of this section.

\subsubsection{A Special Case with Bounded $\ell_2$-norm} First, we give a concentration result bounding the empirical process.
\begin{thm} \label{thm:concentration}
Assume $\phi(\cdot)$ to be Lipschitz continuous with constant $L$, and $\|X\| \leq 1$. Let $(\x_i, \y_i)_{i=1}^n$ be a set of i.i.d.~samples drawn from an unknown distribution $\P_{XY}$. Then with probability at least $1 - \delta$, for every $\w$ with $\|\w - \w_*\| \leq r$, we have
\begin{equation} \label{eqn:bound:1}
\ep(\w) - \ep(\w_*) - \frac{1}{n}\sum_{i=1}^n \left( \phi\left(y_i\w^{\top}\x_i\right) - \phi\left(y_i\w_*^{\top}\x_i\right)\right) \leq  \frac{Lr}{\sqrt{n}} \left( 4  + \sqrt{2\log \frac{m}{\delta}} \right).
\end{equation}
\end{thm}
We observe the upper bound in (\ref{eqn:bound:1}) depends on the radius $r$ of the solution space, and thus the upper bound becomes tither as the solution is approaching to $\w_*$, a key idea used in the local Rademacher complexity~\cite{Local_RC,Local:Vladimir}. Notice that in the case of binary loss, we do not have this nice property.

Similar to Theorem~\ref{thm:induction}, we have the following theorem bounding the number of label requests in each iterations.
\begin{thm} \label{thm:induction-2}
Suppose Assumptions (I)-(III) hold, and $\phi(\cdot)$ is $L$-Lipschitz continuous. Assume $\psi(z) \geq az^{\gamma}$ for any $z \in (0, 1)$ and $\bar{\w}_* \in \Omega_k$. Then, with a probability $1 - \delta/m$, we have $\bar{\w}_* \in \Omega_{k+1}$, if
\begin{equation}\label{eqn:nk:2}
n_k = \left(\frac{\mu \ell_+^{1/\kappa}  2^{\gamma_-+\gamma_+/\kappa} \theta(\varepsilon)}{\ell_-}\right)^{2\gamma} \left(\frac{2LR }{a } \left( 4 + \sqrt{2 \log \frac{m}{\delta}} \right)\right)^2 r_k^{2\left(1 + \frac{\gamma \gamma_+}{\kappa}-\gamma \gamma_-\right)}.
\end{equation}
\end{thm}
Note that the assumption $\psi(z) \geq a z^{\gamma}$ is met by the two examples of $\psi$-transform given in Section~\ref{sec:justificaiton-III}.

Following the same analysis as Theorem~\ref{thm:complexity-1}, we have the following results for the convex case.
\begin{thm} \label{thm:complexity-2}
Suppose Assumption (I)-(IV) hold. Assume $\psi(z) \geq az^{\gamma}$ for any $z \in (0, 1)$. Let $\w_{m+1}$ be the solution output from the proposed algorithm after $m$ iterations, where $m = \lceil \log_2 (2/\varepsilon) \rceil$. Then, with a probability $1 - \delta$, we have
\[
        \left\|\w_{m+1} -\bar{\w}_*  \right\| \leq \varepsilon,
\]
and the total number of labeled instances is bounded by
\[
n \leq \begin{cases}
n_0 \frac{ 2^{2\alpha} }{2^{2\alpha}-1} \left( \frac{4}{\epsilon}\right)^{2\alpha} ,& \alpha > 0\\
 n_0 \log_2 \frac{4}{\varepsilon} , & \alpha \leq 0
\end{cases}
\]
where
\[
\begin{split}
\alpha&= \gamma \gamma_-- \frac{\gamma \gamma_+}{\kappa} -1 , \\
n_0 & = 2^{-4 \alpha}  \left(\frac{\mu \ell_+^{1/\kappa}  2^{\gamma_-+\gamma_+/\kappa} \theta(\varepsilon)}{\ell_-}\right)^{2\gamma} \left(\frac{2LR }{a } \left( 4 + \sqrt{2 \log \frac{m}{\delta}} \right)\right)^2.
\end{split}
\]
\end{thm}

\paragraph{Remark}  As indicated by Theorem~\ref{thm:complexity-2}, an exponential reduction in label complexity can be achieved if $\alpha \leq 0$. More specifically, the number of labeled instances requested by Algorithm~\ref{alg:1} is
\[
O \left( \left(\theta(\varepsilon)\right)^{2\gamma}    \log  \frac{1}{\varepsilon} \right).
\]
Since $\alpha \leq 0$ implies
\[
\kappa \leq \frac{\gamma \gamma_+}{\gamma \gamma_--1} ,
\]
we have an exponential reduction in label complexity even when $\kappa > 1$ provided the above inequality holds. To be more concrete, consider the case when $\P_X$ is orthogonally invariant, we have $\gamma_-=\kappa$ and $\gamma_+=1$, as discussed in Lemmas~\ref{lem:lower} and \ref{lem:upper:new}, and therefore
\[
\kappa \leq \kappa_0 := \frac{ 1 + \sqrt{1+4 \gamma^2 } }{2 \gamma}
\]
will ensure $\alpha \leq 0$ and consequentially an exponential reduction in label complexity. We emphasize that our result does not contradict with the minimax rate of active learning proved in~\cite{Castro:2007:MBA} because we assume the prior knowledge of $R = \|\w_*\|$ is available to the learner.

\subsubsection{A Relaxed Case}
In the following, we study a more general assumption that the $\ell_2$-norm of $X$ is a sub-exponential or a sub-gaussian random variable, i.e., \[
\big \| \|X\| \big \|_{\psi_1} \leq 1 \textrm{ or } \big \| \|X\| \big \|_{\psi_2} \leq 1.
\]
For $\alpha> 0$, the $\psi_\alpha$-norm of a random variable $\eta$, which is a special Orlicz norm~\cite[Section A.1]{Oracle_inequality}, is defined as follows
\[
\|\eta\|_{\psi_\alpha} = \inf\left\{C > 0: \E\left[\exp \left(\left( \frac{|\eta|}{C}\right)^\alpha  \right) \right] \leq 2 \right\}.
\]
The analysis for this case  is almost the same as the previous one, except that we need a generalized version of Theorem~\ref{thm:concentration}.
\begin{thm} \label{thm:con:2}
Assume $\phi(\cdot)$ is Lipschitz continuous with constant $L$, and the marginal distribution $P_{X}$ ensures $\big \| \|X\| \big \|_{\psi_1} \leq 1 $ or $\big \| \|X\| \big \|_{\psi_2} \leq 1$.  Let $(\x_i, \y_i)_{i=1}^n$ be a set of i.i.d.~samples drawn from an unknown distribution $P_{XY}$. Then with probability at least $1 - \delta$, for every $\w$ with $\|\w - \w_*\| \leq r$, we have
\begin{equation} \label{eqn:bound:2}
\ep(\w) - \ep(\w_*) - \frac{1}{n}\sum_{i=1}^n \left( \phi\left(y_i\w^{\top}\x_i\right) - \phi\left(y_i\w_*^{\top}\x_i\right)\right) \leq   \frac{C L r}{\sqrt{n}} \left(1 + \sqrt{\log \frac{m}{\delta}} \right).
\end{equation}
for some constant $C$, provided that
\begin{equation} \label{eqn:n:cond}
n \geq 1+ 2 \log \frac{m}{\delta} \cdot \log \left( \frac{2}{e} \log \frac{m}{\delta}\right).
\end{equation}
\end{thm}
Comparing (\ref{eqn:bound:1}) and (\ref{eqn:bound:2}), we can see the bounds in  Theorems~\ref{thm:concentration} and \ref{thm:con:2} differ only by a constant factor, provided the condition in (\ref{eqn:n:cond}) holds. Thus, we just need to make the following modifications  to Theorems~\ref{thm:induction-2} and \ref{thm:complexity-2}: (i) replacing the factor $4 + \sqrt{2 \log \frac{m}{\delta}}$ with $C\left(1 + \sqrt{\log \frac{m}{\delta}} \right)$, and (ii) adding constraints to ensure (\ref{eqn:n:cond}) is true. Notice that the condition in (\ref{eqn:n:cond}) only requires the number of labeled instances in each iteration to be on the order of $\Omega(\log \log \frac{1}{\epsilon})$. As a result, the total number of labeled instances is on the order of  $\Omega( \log \frac{1}{\epsilon} \cdot \log \log \frac{1}{\epsilon})$, leading to no change on the order of the sample complexity compared to Theorem~\ref{thm:con:2}.
\section{Conclusion}
In this paper, we study the active learning problem with a convex surrogate loss. Our results show that with some additional assumptions, the convex surrogate loss not only improves the computational efficiency of  active learning, but also reduces the sample complexity. In particular, our analysis reveals that it is possible to achieve a exponential reduction in the label complexity even when the noisy level is high.

\bibliographystyle{plain}
\bibliography{active-learning}

\appendix

\section{More Discussion about the Upper Bound in Assumption (III)}
We can have a more general result by exploiting the relationship between the binary excess risk $\ebin(\w) - \ebin(\w_*)$ and the convex excess risk $\ep(\w) - \ep(\w_*)$. Using Theorem~\ref{thm:classification-caliberated}, we have the following result for $\ell_+$ and $\gamma_+$.
\begin{lemma} \label{lem:upper}
Assume (i) $\phi$ is non-negative, (ii) $\ep(\w)$ is $L_{\phi}$-strongly smooth, that is,
\[
\ep(\w')\leq \ep(\w) + \langle \nabla \ep(\w), \w'-\w\rangle + \frac{L_{\phi}}{2} \|\w'-\w\|^2,
\]
 and (iii) $\psi(z) \geq a z^{\gamma}$ for any $z \in [0, 1]$. We have
\[
\ell_+ = \left(  \frac{ L_{\phi} R^2}{2 a} \right)^{1/\gamma}, \textrm{ and } \gamma_+=\frac{2}{\gamma},
\]
where $R =\|\w_*\|$.
\end{lemma}
\begin{proof}
Combining Theorem~\ref{thm:classification-caliberated} and Assumption (II), we have
\[
\psi\left(\ebin(\w) - \ebin(\w_*)\right) \leq \ep(\w) - \ep(\w_*) \leq \frac{L_{\phi}}{2}\|\w - \w_*\|^2.
\]
Since we can arbitrary scale $\w$ without changing its binary risk $\ebin(\w)$, we have
\[
\psi\left(\ebin(\w) - \ebin(\w_*)\right)=\psi\left(\ebin(R \bar{\w}) - \ebin(\w_*)\right) \leq \frac{L_{\phi} R^2}{2}\|\bar{\w} - \bar{\w}_*\|^2.
\]
From the assumption $\psi(z) \geq az^{\gamma}$, we have
\[
a \left(\ebin(\w) - \ebin(\w_*)\right)^{\gamma} \leq \frac{L_{\phi}R^2}{2}\|\bar{\w} - \bar{\w}_*\|^2,
\]
which completes the proof.
\end{proof}
An example of convex loss that satisfies the conditions in Lemma~\ref{lem:upper} is the truncated quadratic loss $\phi(\alpha) = [\max(0, 1 - \alpha)]^2$, which is $1$-strongly smooth with $\psi(z) = z^2$.

\section{Proof of Lemma~\ref{lem:selection}}
We first consider the case $r_k \leq 1$ such that $2 \arcsin(r_k/2) < \frac{\pi}{2}$. Using a simple geometry argument, it is easy to show that (\ref{eqn:selection-criterion}) is equivalent to
\[
\sgn(\w^{\top}\x) \neq \sgn(\w^{\top}_k \x), \ \exists  \w \in \left\{ \w \in \R^d: \theta(\w, \w_k) \leq 2 \arcsin(r_k/2) \right\},
\]
which is  equivalent to
\[
\frac{\pi}{2} -2\arcsin(r_k/2) \leq \theta(\w_k,\x) \leq \frac{\pi}{2} + 2\arcsin(r_k/2),
\]
leading to the following condition
\[
|\cos(\theta(\w_k,\x))| \leq \sin(2\arcsin(r_k/2)).
\]
Using the fact
\[
\sin(\theta)= 2 \sin(\theta/2)\cos(\theta/2)= 2 \sin(\theta/2) \sqrt{1-\sin^2(\theta/2)},
\]
we have
\[
|\bar{\x}^{\top}\w_k|=|\cos(\theta(\w_k,\x))| \leq   r_k \sqrt{1-r^2_k/4}.
\]

If $r_k =2 $, it is obvious that both (\ref{eqn:selection-criterion}) and (\ref{eqn:selection2}) become vacuous.
\section{Proof of Theorem~\ref{thm:induction}}
Let $\Omega_k$ be the subset of hypothesis obtained in epoch $k$ with center $\w_k$ and radius $r_k$, i.e.,
\[
    \Omega_k = \left\{\w \in \R^d: \|\w\|=1, \|\w - \w_k\| \leq r_k\right\}.
\]
By the induction assumption, we have $\bar{\w}_* \in \Omega_k$, and thus for any $\w \in \Omega_k$
\[
\|\w-\bar{\w}_*\| \leq \|\w-\w_k\| + \|\bar{\w}_*-\w_k\| \leq 2 r_k.
\]
Using the upper bound in Assumption (III), for any $\w \in \Omega_k$, we have
\[
 \ebin(\w) - \ebin(\w_*) \leq \ell_+\|\bar{\w} - \bar{\w}_*\|^{\gamma_+} \leq \ell_+ 2^{\gamma_+} r_k^{\gamma_+}.
\]
According to the Tsybakov's low noise condition~\cite{Tsybakov04optimalaggregation}, for any $\w \in \Omega_k$, we have
\[
\Pr_{X \sim \P_X} \left\{\sgn(\w^{\top}X) \neq \sgn(\w_*^{\top}X) \right\} \leq \mu (\ebin(\w) - \ebin(\w_*))^{1/\kappa} \leq \mu \ell_+^{1/\kappa} 2^{\gamma_+/\kappa} r_k^{\gamma_+/\kappa}.
\]
As a result, we have
\[
\Omega_k \subseteq B(\w, \mu \ell_+^{1/\kappa}  2^{\gamma_+/\kappa} r_k^{\gamma_+/\kappa}) \Rightarrow \dis(\Omega_k) \subseteq \dis\left(B\left(\w, \mu \ell_+^{1/\kappa}  2^{\gamma_+/\kappa} r_k^{\gamma_+/\kappa}\right)\right).
\]
From the definition of disagreement coefficient, we have
\begin{equation} \label{eqn:disagree}
\Pr\left(\dis(\Omega_k)\right) \leq \Pr\left(\dis\left(B\left(\w, \mu \ell_+^{1/\kappa}  2^{\gamma_+/\kappa} r_k^{\gamma_+/\kappa}\right)\right)\right) \leq \mu \ell_+^{1/\kappa}  2^{\gamma_+/\kappa} r_k^{\gamma_+/\kappa} \theta(\varepsilon).
\end{equation}

Define $\C_k = \X\setminus \dis(\Omega_k)$. Notice that $\dis(\Omega_k)$ is the subset of instances for which at least two classifiers from $\Omega_k$ will result in different predictions.  Since $\w_* \in \Omega_k$, we have for any given classifier $\w \in \Omega_k$, $\sgn(\w^{\top}\x) = \sgn(\w_*^{\top}\x)$ for all $\x \in \C_k$. Hence, we have, for any $\w \in \Omega_k$,
\begin{equation} \label{eqn:binary}
\ebin(\w) - \ebin(\w_*) = \left(\ebin(\w|\dis(\Omega_k)) - \ebin(\w_*|\dis(\Omega_k))\right)\Pr(\dis(\Omega_k)),
\end{equation}
where $\ebin(\w|T)$ is defined as the binary risk for instances sampled from the set $T$. It is easy to verify that the $n_k$ training instances labeled at epoch $k$ are i.i.d. samples from the space $\dis(\Omega_k)$. To bound the generalization error, we use the following theorem from~\cite{vapnik-1998-statistical}.
\begin{thm}
For any distribution $\D$ over $\X \times \{\pm 1\}$, with a probability at least $1 - \delta$ over the $m$ i.i.d.~samples $(X_i, Y_i)_{i=1}^n$ from $\D$, for every $h \in \mathcal{H}$, we have
\[
|er_Z(h) - er_D(h)| \leq G(m, \delta),
\]
where
\[
er_D(h) = \E_{(X, Y) \sim D}\left[ \ind( Yh(X) \leq 0) \right], \quad er_Z(h) = \frac{1}{m}\sum_{i=1}^m \ind(Y_ih(X) \leq 0),
\]
and
\[
G(m, \delta) = \frac{1}{m} + \sqrt{\frac{\log(4/\delta) + d\log(2em/d)}{m}}.
\]
Here $d$ is the VC dimension of the function space $\H$.
\end{thm}
Since the VC dimension of the linear classifier in $\R^d$ is $d+1$, we have, with a probability $1 - \delta /m$,
\begin{equation} \label{eqn:bin:con}
\ebin(\w|\dis(\Omega_k)) - \ebin(\w_*|\dis(\Omega_k)) \leq \frac{2}{n_k} + 2\sqrt{\frac{\log(4m/\delta) + (d+1)\log(2en_k/(d+1))}{n_k}},
\end{equation}
and therefore
\[
\ebin(\w_{k+1}) - \ebin(\w_*) \overset{\text{(\ref{eqn:disagree}), (\ref{eqn:binary}), (\ref{eqn:bin:con})}}{\leq}  \eta_k r_k^{\gamma_+/\kappa},
\]
where
\[
\eta_k := \mu \ell_+^{1/\kappa}  2^{1+\gamma_+/\kappa} \theta(\varepsilon) \left( \frac{1}{n_k} + \sqrt{\frac{\log(4m/\delta) + (d+1)\log(2en_k/(d+1))}{n_k}}\right).
\]
Using the lower bound in Assumption (III), we have
\[
\ebin(\w_{k+1}) - \ebin(\w_*) \geq \ell_-\left\|\w_{k+1} - \bar{\w}_*\right\|^{\gamma_-}.
\]
As a result, with a probability $1 - \delta/m$, we have
\begin{equation} \label{eqn:distance1}
\left\|\w_{k+1} - \bar{\w}_*\right\| \leq \frac{\eta^{1/\gamma_-}_k}{\ell^{1/\gamma_-}_-} r_k^{\gamma_+/(\gamma_-\kappa)}.
\end{equation}
To ensure the R.H.S.~of (\ref{eqn:distance1}) is smaller than $r_{k+1} = r_k/2$, we need
\[
\eta_k \leq \frac{\ell_-}{2^{\gamma_-}}r_k^{\gamma_- - \gamma_+/\kappa},
\]
which is
\begin{equation} \label{eqn:cond1}
\frac{1}{n_k} + \sqrt{\frac{\log(4m/\delta) + (d+1)\log(2en_k/(d+1))}{n_k}}  \leq  \frac{\ell_- r_k^{\gamma_- - \gamma_+/\kappa} }{\mu \ell_+^{1/\kappa}  2^{1+\gamma_-+\gamma_+/\kappa}\theta(\varepsilon) }.
\end{equation}
Since $(d+1)\log(2en_k/(d+1)) \geq 1$, we must have
\[
\frac{1}{n_k} \leq \sqrt{\frac{\log(4m/\delta) + (d+1)\log(2en_k/(d+1))}{n_k}}.
\]
Thus, to satisfy the condition in (\ref{eqn:cond1}), it is sufficient to ensure
\begin{equation} \label{eqn:cond2}
n_k \geq \left(   \frac{ \mu \ell_+^{1/\kappa}  2^{2+\gamma_-+\gamma_+/\kappa} \theta(\varepsilon)}{\ell_- r_k^{\gamma_- - \gamma_+/\kappa} }\right)^2  \left( \log \frac{4m}{\delta} + (d+1)\log \frac{2en_k}{d+1}\right).
\end{equation}
A sufficient condition to (\ref{eqn:cond2}) is
\begin{equation} \label{eqn:cond3}
n_k \geq 2 \left(   \frac{c \theta(\varepsilon)  }{ r_k^{\gamma_- - \gamma_+/\kappa} }\right)^2   \log \frac{4m}{\delta},
\end{equation}
and
\begin{equation} \label{eqn:cond4}
n_k \geq 2 \left(   \frac{c \theta(\varepsilon)  }{r_k^{\gamma_- - \gamma_+/\kappa} }\right)^2 (d+1)\log \frac{2en_k}{d+1},
\end{equation}
where
\[
c=\frac{\mu \ell_+^{1/\kappa}  2^{2+\gamma_-+\gamma_+/\kappa}}{\ell_-}.
\]

To address the inequality in (\ref{eqn:cond4}), define
\[
a=2 \left(   \frac{c \theta(\varepsilon) }{r_k^{\gamma_- - \gamma_+/\kappa} }\right)^2.
\]
Notice that
\[
\log \frac{2en_k}{d+1}=  \log \frac{n_k}{2a(d+1)}   + \log 4 a e \leq \frac{n_k}{2a(d+1)} -1 + \log 4 a e  = \frac{n_k}{2a(d+1)} + \log 4 a,
\]
where we use the inequality $1+ \log x \leq x$. Thus, a sufficient condition to (\ref{eqn:cond4}) is
\[
n_k \geq a (d+1) \left( \frac{n_k}{2a(d+1)}  + \log 4 a \right),
\]
which implies
\begin{equation}\label{eqn:cond5}
n_k \geq 2a(d+1) \log 4a= 4\left(   \frac{c \theta(\varepsilon) }{r_k^{\gamma_- - \gamma_+/\kappa} }\right)^2 (d+1) \left( \log 8 + 2 \log   \frac{c \theta(\varepsilon)  }{ r_k^{\gamma_- - \gamma_+/\kappa} } \right).
\end{equation}
Then, it is clear that (\ref{eqn:nk}) is a sufficient condition to ensure both (\ref{eqn:cond3}) and (\ref{eqn:cond5}).

\section{Proof of Theorem~\ref{thm:concentration}}
Define
\[
P_n(\w) = \frac{1}{n}\sum_{i=1}^n \phi(y_i\w^{\top}\x_i) - \phi(y_i\w_*^{\top}\x_i), \quad P(\w) = \E \left[\phi\left(y\w^{\top}\x\right) - \phi\left(y\w_*^{\top}\x\right)\right],
\]
and
\[
\|P_n(\w) - P(\w)\|_r = \sup\left\{|P_n(\w) - P(\w)|: \|\w - \w_*\| \leq r \right\}.
\]
Since
\[
    |\phi(y \w^{\top}\x ) - \phi(y \w_*^{\top}\x )| \leq L \|\w- \w_*\| \|\x \| \leq L r,
\]
following \cite[Theorem 2.5]{Oracle_inequality}, we have, with a probability $1 - \delta/m$,
\begin{equation} \label{eqn:oracle}
\|P_n(\w) - P(\w)\|_r \leq  Lr\sqrt{\frac{2\log m/\delta}{n}} + \E\left[\|P_n(\w) - P(\w)\|_r\right].
\end{equation}
Let $\{\varepsilon_i\}$ be Rademacher random variables, that is, $\varepsilon_i$ takes the values $1$ and $-1$ with probability $1/2$ each. Using symmetrization inequality of Rademacher complexity~\cite{Oracle_inequality}, we have
\[
\E\left[\|P_n(\w) - P(\w)\|_r\right] \leq 2 \E \left[\frac{1}{n}\sup_{\|\w - \w_*\| \leq r} \left| \sum_{i=1}^n \varepsilon_i \left(\phi(y_i\w^{\top}\x_i) - \phi(y_i\w_*^{\top}\x_i)\right)\right| \right].
\]
Since $\phi(\cdot)$ is $L$-Lipschitz continuous, one can use the contraction inequality~\cite{Oracle_inequality} to get
\begin{equation} \label{eqn:rade}
\begin{split}
& \E\left[\|P_n(\w) - P(\w)\|_r\right]  \leq \frac{4L}{n} \E\left[\sup_{\|\w - \w_*\| \leq r} \left|\sum_{i=1}^n \varepsilon_i (\w - \w_*)^{\top}\x_i \right| \right] \\
\leq & \frac{4Lr}{n} \E\left[ \left\|\sum_{i=1}^n \varepsilon_i \x_i \right\| \right] \leq \frac{4Lr}{n} \sqrt{\E\left[ \left\|\sum_{i=1}^n \varepsilon_i \x_i \right \|^2 \right]}   \\
= & \frac{4Lr}{n} \sqrt{\E\left[ \sum_{i=1}^n \|\x_i \|^2 + \sum_{i \neq j} \varepsilon_i \varepsilon_j \x_i^\top \x_j \right]} \leq  \frac{4Lr}{n} \sqrt{n} = \frac{4Lr}{\sqrt{n}},
\end{split}
\end{equation}
where in the last inequality we use the fact that $\varepsilon_i$'s are independent from $\x_i$'s such that
\[
\E \varepsilon_i \varepsilon_j \x_i^\top \x_j= \E \left[\x_i^\top \x_j \E [\varepsilon_i \varepsilon_j] \right]= 0, \textrm{ when } i\neq j.
\]
We complete the proof by combining (\ref{eqn:oracle}) and (\ref{eqn:rade}).
\section{Proof of Theorem~\ref{thm:induction-2}}
Based on our induction assumption, we have $\|\bar{\w}_* -\w_k \| \leq r_k$, implying $\|\w_*- R\w_k \| \leq R r_k$. Since $\tilde{\w}_{k+1}$ is the optimal solution to (\ref{eqn:convex-update}) and $\w_* \in \Delta_k$, we have
\begin{equation} \label{eqn:min}
 \sum_{t=1}^{n_k} \phi \left(y_k^t \tilde{\w}_{k+1}^{\top}\x_k^t \right) -  \sum_{t=1}^{n_k} \phi \left(y_k^t \w_*^{\top}\x_k^t\right) \leq 0.
\end{equation}

Notice that for any $\w \in \Delta_k$, we have $\|\w -\w_*\|  \leq 2 R r_k$. Following Theorem~\ref{thm:concentration} and (\ref{eqn:min}), we have, with a probability $1 - \delta/m$
\begin{equation} \label{eqn:bound:convex}
\ep(\tilde{\w}_{k+1}|\dis(\Omega_k)) - \ep(\w_*|\dis(\Omega_k)) \leq  \frac{2LRr_k}{\sqrt{n_k}} \left( 4 + \sqrt{2\log \frac{m}{\delta}}, \right)
\end{equation}
where $\ep(\cdot|T)$ is defined as the convex risk for instances sampled from the set $T$. From the discussion in the end of Section~\ref{eqn:Ass:II}, we know that $\w_*$ minimizes both the $\ebin(\cdot|\dis(\Omega_k))$ and $\ep(\cdot|\dis(\Omega_k))$ over all measurement functions. Thus, we can apply Theorem~\ref{thm:classification-caliberated} to bound the excess binary risk as follows
\[
\begin{split}
& \psi\big(\ebin(\tilde{\w}_{k+1}|\dis(\Omega_k)) - \ebin(\w_*|\dis(\Omega_k)) \big) \\
\leq  & \ep(\tilde{\w}_{k+1}|\dis(\Omega_k)) - \ep(\w_*|\dis(\Omega_k)) \\
\overset{\text{(\ref{eqn:bound:convex})}}{\leq}  &    \frac{2LRr_k}{\sqrt{n_k}} \left( 4 + \sqrt{2\log \frac{m}{\delta} } \right).
\end{split}
\]
Using the assumption that $\psi(z) \geq a z^{\gamma}$, we have
\begin{equation} \label{eqn:bin:con:2}
\ebin(\tilde{\w}_{k+1}|\dis(\Omega_k)) - \ebin(\w_*|\dis(\Omega_k)) \leq  \left(\frac{2LR }{a \sqrt{n_k}} \left( 4 + \sqrt{2\log \frac{m}{\delta}} \right)\right)^{1/\gamma}   r_k^{1/\gamma}.
\end{equation}

Following the same analysis as that for Theorem~\ref{thm:complexity-1}, we have, with a probability $1 - \delta$
\[
\ebin(\w_{k+1}) - \ebin(\w_*) \overset{\text{(\ref{eqn:disagree}), (\ref{eqn:binary}), (\ref{eqn:bin:con:2})}}{\leq} \nu_k r_k^{\frac{1}{\gamma} + \frac{\gamma_+}{\kappa}},
\]
where
\[
\nu_k = \mu \ell_+^{1/\kappa}  2^{\gamma_+/\kappa} \theta(\varepsilon) \left(\frac{2LR }{a \sqrt{n_k}} \left( 4 + \sqrt{2\log \frac{m}{\delta}} \right)\right)^{1/\gamma}
\]
Using the lower bound in Assumption (III), we have
\begin{equation} \label{eqn:distance2}
\left\|\w_{k+1} - \bar{\w}_*\right\| \leq \frac{\nu^{1/\gamma_-}_k}{\ell^{1/\gamma_-}_-}r_k^{\frac{1}{\gamma_-}\left(\frac{1}{\gamma} + \frac{\gamma_+}{\kappa}\right)}.
\end{equation}
To ensure the R.H.S.~of (\ref{eqn:distance2}) is smaller than $r_{k+1} = r_k/2$, we need
\[
\nu_k \leq \frac{\ell_-}{2^{\gamma_-}} r_k^{\gamma_--\frac{1}{\gamma} - \frac{\gamma_+}{\kappa}}.
\]
which requires $n_k$ to satisfy the condition in (\ref{eqn:nk:2}).
\section{Proof of Theorem~\ref{thm:con:2}}
Similar to the proof of Theorem~\ref{thm:concentration}, we define
\[
P_n(\w) = \frac{1}{n}\sum_{i=1}^n \phi(y_i\w^{\top}\x_i) - \phi(y_i\w_*^{\top}\x_i), \quad P(\w) = \E \left[\phi\left(y\w^{\top}\x\right) - \phi\left(y\w_*^{\top}\x\right)\right],
\]
and
\[
\|P_n(\w) - P(\w)\|_r = \sup\left\{|P_n(\w) - P(\w)|: \|\w - \w_*\| \leq r \right\}.
\]
The difference is that we need to use the Adamczak bound~\cite[Section 2.3]{Oracle_inequality} to deal with the challenge that the function value may be unbounded. Based on the Adamczak bound, we have, with a probability $1 - \delta/m$,
\begin{equation}  \label{eqn:adam:1}
\|P_n(\w) - P(\w)\|_r \leq  C_1 \left[\E\left[ \|P_n(\w) - P(\w)\|_r \right] + \sigma_r \sqrt{\frac{\log (m/\delta)}{n}} + U_r \frac{\log (m/\delta)}{n} \right],
\end{equation}
where
\[
\begin{split}
\sigma_r^2 & \leq \sup_{\w: \|\w - \w_*\| \leq r} \E \left[ \left(\phi \left(y \w^{\top}\x \right) - \phi\left(y  \w_*^{\top}\x\right )  \right)^2 \right], \\
U_r & =\left\|\max_{1 \leq i \leq n} \sup_{\w:\|\w - \w_*\| \leq r} \left|\phi(y_i\w^{\top}\x_i) - \phi(y_i\w_*^{\top}\x_i)\right| \right\|_{\psi_1}, \\
\end{split}
\]
and $C_1$ is some constant.  In the following, we consider how to bound the three terms on the R.H.S.~of (\ref{eqn:adam:1}).

\paragraph{Bounding $\E\left[ \|P_n(\w) - P(\w)\|_r \right]$} Following the same analysis in the proof of Theorem~\ref{thm:concentration}, we arrive at
\begin{equation} \label{eqn:adam:3}
\E\left[\|P_n(\w) - P(\w)\|_r \right] \leq \frac{4Lr}{n} \sqrt{\E\left[ \sum_{i=1}^n \|\x_i \|^2 \right]}.
\end{equation}
Following the equivalence of sub-exponential (sub-gaussian) properties~\cite{Vershynin_RM}, we have
\begin{equation} \label{eqn:adam:4}
\E\left[  \|\x_i \|^2 \right] \leq  \begin{cases}
\left( C_2 \big\| \|\x_i \| \big \|_{\psi_1} 2 \right)^2 \leq  4 C_2^2, &\textrm{if } \big \| \|X\| \big \|_{\psi_1} \leq 1\\
\left( C_3 \big\| \|\x_i \| \big\|_{\psi_2} \sqrt{2} \right)^2 \leq  2 C_3^2, & \textrm{if } \big \| \|X\| \big \|_{\psi_2} \leq 1
\end{cases}
\end{equation}
where $C_2$ and $C_3$ are some constants. Thus, in both cases, we have $\E\left[  \|\x_i \|^2 \right] \leq C_4$ for some constant $C_4$. Combining (\ref{eqn:adam:3}) and (\ref{eqn:adam:4}), we have
\begin{equation} \label{eqn:adam:5}
\E\left[\|P_n(\w) - P(\w)\|_r \right] \leq \frac{C_5 Lr}{\sqrt{n}},
\end{equation}
for some constant $C_5$.
\paragraph{Bounding $\sigma_r^2 $}  Since $\phi(\cdot)$ is Lipschitz continuous with constant $L$, we have
\begin{equation} \label{eqn:adam:6}
\begin{split}
\sigma_r^2  & \leq \sup_{\w: \|\w - \w_*\| \leq r} \E \left[ \left( L y (\w-\w_*)^{\top}\x   \right)^2 \right] \\
 & \leq  L^2 \sup_{\w: \|\w - \w_*\| \leq r} \E \left[  \|\w-\w_*\|^2 \|\x\|^2   \right] \leq r^2 L^2 \E \left[ \|\x\|^2   \right] \overset{\text{(\ref{eqn:adam:4})}}{\leq} C_6 r^2 L^2,
\end{split}
\end{equation}
for some constant $C_6$.
\paragraph{Bounding $U_r$}
Based on the bound for the Orlicz norm of a finite maximum~\cite[Lemma 2.2.2]{Empirical:Process}, we have
\[
\begin{split}
U_r \leq & C_7 \sqrt{\log (n + 1)} \max_{1 \leq i \leq n} \left\| \sup_{\w:\|\w - \w_*\| \leq r} \left|\phi(y_i\w^{\top}\x_i) - \phi(y_i\w_*^{\top}\x_i)\right| \right\|_{\psi_1} \\
= & C_7 \sqrt{\log (n + 1)} \left\| \sup_{\w:\|\w - \w_*\| \leq r} \left| \phi(y_1\w^{\top}\x_1) - \phi(y_1\w_*^{\top}\x_1) \right| \right\|_{\psi_1} ,
\end{split}
\]
where $C_7$ is some constant, and the last equality comes from the fact $\{(\x_i,\y_i)\}$ follow the same distribution. Based on the Lipschitz continuity of the loss function $\psi(\cdot)$, we further have
\begin{equation} \label{eqn:adam:7}
\begin{split}
U_r  \leq & C_7 \sqrt{\log (n + 1)} \left\|\sup\limits_{\w:\|\w - \w_*\| \leq r} \left| L (\w- \w_*)^{\top}\x_1 \right|\right\|_{\psi_1} \leq  C_7   L r   \sqrt{\log (n + 1)} \big \| \| \x_1 \| \big \|_{\psi_1}.
\end{split}
\end{equation}

Thus, in the case that $\big \| \|X\| \big \|_{\psi_1} \leq 1$, we have \
\[
U_r \leq  C_7  L r   \sqrt{\log (n + 1)}.
\]
In the case that $\big \| \|X\| \big \|_{\psi_2} \leq 1$, we use the following relation between $\psi_p$- and $\psi_q$-norms~\cite[Page 95]{Empirical:Process}
\begin{equation} \label{eqn:adam:8}
\|\eta\|_{\psi_p} \leq \| \eta \|_{\psi_q} \left( \log 2 \right)^{1/q-1/p}, \ p \leq q.
\end{equation}
Then, we have
\[
U_r   \overset{\text{(\ref{eqn:adam:7}),(\ref{eqn:adam:8})}}{\leq}    C_7 \left( \log 2 \right)^{-1/2}   L r   \sqrt{\log (n + 1)} \big \| \| \x_1 \| \big \|_{\psi_2} \leq  C_7 \left( \log 2 \right)^{-1/2}  L r   \sqrt{\log (n + 1)}.
\]
As a result, in both cases, we have
\begin{equation} \label{eqn:adam:8}
U_r \leq C_8 L r \sqrt{\log (n + 1)},
\end{equation}
for some constant $C_8$.

Substituting (\ref{eqn:adam:5}), (\ref{eqn:adam:6}), and (\ref{eqn:adam:8})  into (\ref{eqn:adam:1}), we have, with a probability $1 - \delta/m$,
\begin{equation} \label{eqn:adam:8}
\|P_n(\w) - P(\w)\|_r \leq  C_9 L r\left[\frac{1}{\sqrt{n}} + \sqrt{\frac{\log (m/\delta)}{n}} \left(1 +  \sqrt{ \log (n + 1) \frac{\log (m/\delta)}{n}} \right) \right],
\end{equation}
for some constant $C_9$. From the condition in (\ref{eqn:n:cond}),  we have
\begin{equation} \label{eqn:adam:9}
\begin{split}
n \geq &  \frac{n+1}{2} + \log \frac{m}{\delta} \cdot \log \left( \frac{2}{e} \log \frac{m}{\delta}\right) \\
 = & \left( \frac{n+1}{2 \log (m/\delta)}  +  \log \left( \frac{2}{e} \log \frac{m}{\delta}\right) \right) \log \frac{m}{\delta}\\
 \geq & \left( 1 + \log \left( \frac{n+1}{2 \log (m/\delta)}\right) +  \log \left( \frac{2}{e} \log \frac{m}{\delta}\right) \right) \log \frac{m}{\delta} \\
 =& \log (n+1) \log \frac{m}{\delta},
\end{split}
\end{equation}
where in the second inequality we use the fact that $1+ \log x \leq x$. Substituting (\ref{eqn:adam:9})   into (\ref{eqn:adam:8}), we have, with a probability $1 - \delta/m$,
\[
\|P_n(\w) - P(\w)\|_r \leq  C Lr \left(\frac{1}{\sqrt{n}} + \sqrt{\frac{\log (m/\delta)}{n}} \right),
\]
for some constant $C$.
\end{document}